\documentclass[11pt]{article}

\usepackage{amsmath,amssymb,amsthm,amsfonts}
\usepackage[dvipsnames]{xcolor}
\usepackage{tikz}
\usepackage{hyperref}
\usepackage{enumitem}
\usepackage{thmtools}
\usepackage[ruled,noline,noend]{algorithm2e}

\usepackage{cleveref}
\usepackage{graphicx}
\usepackage{fullpage}
\usepackage{natbib}
\newtheorem{theorem}{Theorem}[section]
\newtheorem{lemma}[theorem]{Lemma}

\newtheorem{corollary}[theorem]{Corollary}
\newtheorem{definition}[theorem]{Definition}
\newtheorem{remark}[theorem]{Remark}
\newcommand{\E}{\mathbb{E}}
\newcommand{\Prb}{\mathbb{P}}
\newcommand{\1}{\mathbf{1}}
\newcommand{\X}{\mathcal{X}}
\newcommand{\Lcal}{\mathcal{L}}
\newcommand{\C}{\mathcal{C}}
\newcommand{\D}{\mathcal{D}}

\newcommand{\Vcal}{\mathcal{V}}
\newcommand{\VCdim}{\mathrm{VCdim}}
\newcommand{\errmis}{err_{\mathrm{mis}}}
\newcommand{\errabs}{err_{\mathrm{abs}}}
\newcommand{\sep}{:}
\DeclareMathOperator{\sign}{sign}
\NewDocumentCommand{\good}{O{m} O{c} O{k}}{
    $(#1,#2,#3)$-resilient
}

\crefname{protocol}{Protocol}{Protocols}
\Crefname{protocol}{Protocol}{Protocols}

\title{Reliable Abstention under Adversarial Injections:\\ Tight Lower Bounds and New Upper Bounds}
\author{Ezra Edelman, Surbhi Goel\\ \\ 
University of Pennsylvania}

\begin{document}

\maketitle

\begin{abstract}%
We study online learning in the adversarial injection model introduced by \citet{goel2024adversarialresiliencesequentialprediction}, where a stream of labeled examples is predominantly drawn i.i.d.\ from an unknown distribution $\D$, but may be interspersed with adversarially chosen instances without the learner knowing which rounds are adversarial. Crucially, labels are always consistent with a fixed target concept (the clean-label setting). The learner is additionally allowed to abstain from predicting, and the total error counts the mistakes whenever the learner decides to predict and incorrect abstentions when it abstains on i.i.d.\ rounds. Perhaps surprisingly, prior work shows that oracle access to the underlying distribution yields $O(d^2 \log T)$ combined error for VC dimension $d$, while distribution-agnostic algorithms achieve only $\tilde{O}(\sqrt{T})$ for restricted classes, leaving open whether this gap is fundamental.

We resolve this question by proving a matching $\Omega(\sqrt{T})$ lower bound for VC dimension $1$, establishing a sharp separation between the two information regimes. On the algorithmic side, we introduce a potential-based framework driven by \emph{robust witnesses}, small subsets of labeled examples that certify predictions while remaining resilient to adversarial contamination. We instantiate this framework using two combinatorial dimensions: (1) \emph{inference dimension}, yielding combined error $\tilde{O}(T^{1-1/k})$ for classes of inference dimension $k$, and (2) \emph{certificate dimension}, a new relaxation we introduce. As an application, we show that halfspaces in $\mathbb{R}^2$ have certificate dimension $3$, obtaining the first distribution-agnostic bound of $\tilde{O}(T^{2/3})$ for this class. This is notable since \citet{blum2021robustlearningcleanlabelattack} showed halfspaces are not robustly learnable under clean-label attacks without abstention.

\end{abstract}

\section{Introduction}
Online learning is often analyzed at two extremes.
In the fully adversarial model, instances may be chosen arbitrarily, and the learnability of a realizable binary class is governed by Littlestone dimension.
In the stochastic model, instances are drawn i.i.d.\ from a fixed distribution $\D$, and rates are governed by VC dimension.
Many real streams fall between these: the data may be predominantly stochastic, but occasionally contain adaptively chosen, potentially adversarial examples.
A basic question is what guarantees are possible in such intermediate settings.

We study this question in the \emph{adversarial injection} model of \citet{goel2024adversarialresiliencesequentialprediction}.
There is an unknown target concept $c^\star\in\C$ and an unknown distribution $\D$ over $\X$.
On each round $t=1,\dots,T$, an adversary chooses whether the instance $x_t$ is
(i) drawn from $\D$ (an i.i.d.\ round), or
(ii) injected adversarially (the adversary selects an arbitrary $x_t\in\X$). We write $q_t\in\{0,1\}$ for this hidden indicator, where $q_t=0$ denotes an i.i.d.\ round and $q_t=1$ denotes an injected round.
In both cases the label is always clean, that is, after the learner predicts, it observes $y_t = c^\star(x_t)$.
After seeing $x_t$, the learner may predict $\hat y_t\in\{\pm1\}$ or abstain ($\perp$), and then observes $y_t$.
Crucially, the learner is never told whether round $t$ was i.i.d.\ or injected.

The objective reflects this asymmetry. Mistakes are always counted, while abstentions are counted only on i.i.d.\ rounds, that is, we measure the \emph{combined error}
\[
\mathrm{err}_T = \underbrace{\sum_{t=1}^T \1\bigl[\hat y_t\neq y_t\ \wedge\ \hat y_t\neq \perp\bigr]}_{\text{misclassifications}} + \underbrace{\sum_{t=1}^T \1\bigl[\hat y_t=\perp\ \wedge\ q_t=0\bigr]}_{\text{incorrect abstentions}}.
\]
and study $\E[\mathrm{err}_T]$. This creates a fundamental tension: abstaining is safe against injected points (it avoids misclassification), but abstaining on i.i.d.\ points incurs charged error, on the other hand, predicting too aggressively reduces charged abstentions but risks mistakes on injected points that the learner cannot identify.

In this setting, \citet{goel2024adversarialresiliencesequentialprediction} distinguish two information regimes.
In the \emph{known-$\D$} regime, the learner has access\footnote{They assume a very strong notion of access where the learner can query any function of $\D$ exactly. The queries they assume access to would require $T^{O(d)}$ samples naively.} to $\D$ and can quantify distributional uncertainty (e.g.\ disagreement under $\D$), yielding combined error bound  $O(d^2\log T)$ for classes of VC dimension $d$.
In the more realistic \emph{distribution-agnostic} regime, $\D$ is hidden and the learner must reason from a single mixed stream.
Prior work obtains $\tilde O(\sqrt{d T})$ combined error for VC dimension~$1$ classes and for axis-aligned rectangles, but leaves open whether the large gap between the two regimes is unavoidable. This motivates our central question:
\begin{center}
\emph{Is oracle access to $\D$ fundamentally necessary for strong guarantees, or can a distribution-agnostic learner achieve stochastic-like rates from a single mixed stream?}
\end{center}

\subsection{Our contributions}

\paragraph{A sharp separation even at VC dimension $1$.}
Our first contribution is to show that the gap between the two regimes is real.
We prove that even for VC dimension~$1$, every distribution-agnostic learner must incur $\Omega(\sqrt{T})$ expected combined error.
Together with the $\tilde O(\sqrt{T})$ upper bound of \citet{goel2024adversarialresiliencesequentialprediction}, this pins down the optimal rate for VC-1 in the distribution-agnostic regime and demonstrates a genuine advantage to knowing $\D$.

\begin{theorem}[Lower bound for VC-1, informal]\label{thm:lowerbound informal}
There exists a concept class $\C$ with $\VCdim(\C)=1$ such that for every horizon $T$, there exists an (oblivious) adversary such that every distribution-agnostic learner incurs expected combined error $\Omega(\sqrt{T})$.
\end{theorem}

\paragraph{A potential framework for distribution-agnostic upper bounds.}
For upper bounds, we introduce a potential-based framework that unifies and extends the distribution-agnostic algorithms of prior work.
At a high level, the framework is driven by \emph{robust witnesses}:
small subsets $U$ of previously seen labeled examples such that, for some additional point $x$, assigning $x$ the wrong label causes a guaranteed drop in an underlying score function $f$.
We design a learner that maintains a potential aggregating these scores over the history.
The learner predicts only when one of the two possible labels for the current point would cause a large potential drop, and otherwise abstains.
Under sufficient conditions on the score function $f$, this yields a worst-case bound on mistakes (since each mistake triggers a guaranteed potential drop) and an expected bound on charged abstentions.
The abstention bound relies on the notion of \emph{attackability}, a sample is attackable if some adversarial extension of the history would cause the learner to abstain on it. A key property of $f$, called robustness, ensures that only a small number of samples in any set can be attackable and then exchangeability implies that each i.i.d.\ sample has low probability of being abstained upon, once enough i.i.d. samples exist in the stream.

\paragraph{Combinatorial instantiations: inference and certificate dimension.}
We instantiate this framework via two combinatorial measures. First, we show that finite \emph{inference dimension} \cite{kane2017activeclassificationcomparisonqueries} implies the existence of such a score function $f$ giving us,
\begin{theorem}[Inference dimension, informal]\label{thm: inference informal}
If a class has inference dimension $k$, then there is a distribution-agnostic learner with expected combined error $\tilde{\mathcal{O}}(T^{1-1/k})$.
\end{theorem}
Inference dimension is a structural notion originating in active learning that captures when a small set of examples can certify the labels of other examples.

Second, we introduce \emph{certificate dimension}, a weaker notion than inference dimension that relaxes requirements tied to real-valued representations. In our certificate-dimension instantiation, these witnesses correspond to genuine certificates: a small $U$ whose certificate value determines the label of $x$ among hypotheses consistent with the rest of the data. Using certificate dimension, we obtain the first distribution-agnostic upper bound for halfspaces in this model:
\begin{theorem}[Halfspaces in $\mathbb{R}^2$, informal]
Halfspaces in $\mathbb{R}^2$ have certificate dimension~$3$, yielding expected combined error $\tilde{\mathcal{O}}(T^{2/3})$.
\end{theorem}

This result is particularly striking in light of prior work on clean-label attacks \citep{blum2021robustlearningcleanlabelattack} that showed that halfspaces (in $\mathbb{R}^2$) are \emph{not} robustly learnable under clean-label poisoning without margin assumptions, and even with positive margin, robust learning requires sample complexity exponential in the dimension. Our result demonstrates that abstention fundamentally changes the landscape, the adversary can no longer succeed just by using many adversarial points to attack one sample.

\subsection{Related Work}
\paragraph{Adversarial Injection Model and related models.} Our work builds on \cite{goel2024adversarialresiliencesequentialprediction}, which introduced the adversarial injection model and provided the first upper bounds in both the known-$\D$ and distribution-agnostic regimes. Recently, \cite{heinzler2025adversarialresiliencecleanlabelattacks} gave new results in the agnostic variant of the model. This adversarial injection model can be viewed as the online version of PQ learning \citep{NEURIPS2020_b6c8cf4c,kalai2021efficient, kalai2021optimally,goel2024tolerantalgorithmslearningarbitrary} where the learner receives labeled training samples from a source distribution $P$ and unlabeled samples from a (potentially adversarial) test distribution $Q$, and produces a selective classifier which (1) must make low error on instances it chooses to predict, and (2) not abstain on samples from $P$. Another closely related model is the TDS (Testable Distribution Shift) model \cite{klivans2023testable,klivans2024learning,chandrasekaran2024efficient, goel2024tolerantalgorithmslearningarbitrary} which allows the learner to entirely reject the test set if it is not close to $P$. In these settings the learner always has access to a clean train dataset, and a large test set, whereas in the online setting the train set so far can contain adversarial samples, and we must predict or abstain on a single test sample. This model also broadly relates to work in selective classification and abstention \cite{chowopt,statsabs1, statsabs2, statsabs3,zhang2016extended, cortes2019online, neu2020fast,li2008knows,sayedi2010trading}. We refer the reader to \cite{goel2024adversarialresiliencesequentialprediction} for a discussion on this.

\paragraph{Clean-label attacks and exchangeability.}
\cite{shafahi2018poison,blum2021robustlearningcleanlabelattack} study offline robust learning under clean-label data poisoning, where an attacker adds correctly labeled samples to force mistakes on a specific test point. In this setting, halfspaces even in $\mathbb{R}^2$ are not robustly learnable. The critical difference is that their attacker adds many adversarial samples to attack a specific test sample. In our setting, by contrast, if the attacker uses many adversarial samples to attack a single iid sample, the learner's overall error can still be low. As long as the learner either abstains or correctly predicts the labels for the injected adversarial samples, misclassifying the targeted test sample has minimal impact.
\cite{larsen2026learning} study a weaker setting where the adversary inserts additional correctly labeled points after observing an i.i.d.\ sample, and the learner is evaluated by test error on a fresh i.i.d.\ draw. They show that optimal PAC learners like One-Inclusion Graph and majority voting can be made suboptimal, while ERM remains robust. However, in our setting ERM is not enough, since we have to make predictions/abstentions on the injected rounds as well. 

\paragraph{Active learning with comparison queries.} Inference dimension originates in active learning with comparison queries \citep{kane2017activeclassificationcomparisonqueries}, where it characterizes when a small set of labeled samples can certify the labels of other samples. Subsequent work studies the power of comparisons for learning linear separators \citep{hopkins2020powercomparisonsactivelylearning} and \emph{reliable} active classification \citep{pmlr-v125-hopkins20a}, where the learner may abstain on uncertain samples, a paradigm closely related to ours. While these works focus on active learning with oracle access to labels and comparisons, we show that inference dimension also allows learnability in our passive, distribution-agnostic setting, where the learner cannot choose which samples to query by helping certify labels for other i.i.d. samples. \cite{kane2018generalizedcomparisontreespointlocation} generalizes comparison queries, allowing for an infinite class of queries comparing two samples, allowing them to break lower bounds set in the prior work. Unfortunately, in our setting, the finite number of queries is essential to achieve non-trivial bounds. \cite{kontonis2024activelearningsimplequestions} also introduces more generalized queries to active learning, allowing the learner to ask whether a set of samples all have a certain label, achieving much lower query bounds. In our setting, while the learner's error bound can depend on the number distinct possible answers to queries, these more generalized queries are not directly applicable.

\section{Preliminaries}\label{sec:prelims}
\paragraph{Notation.}
We denote the domain by $\X$ and consider a concept class $\C\subseteq \{\pm1\}^{\X}$. We denote the VC dimension of the class by $d$. We work in the realizable setting where labels are generated by an unknown target $c^\star\in\C$.
For a positive integer $n$, we write $[n]:=\{1,\ldots,n\}$.
Let $\Lcal=\X\times \{\pm 1\}$ denote the space of labeled examples.
For any finite $S\subseteq \Lcal$, the \emph{version space} is
$\Vcal(S):=\{h\in \C \sep h(x)=y \text{ for all } (x,y)\in S\}$.
For any $V\subseteq \C$, $x\in \X$, and $\ell\in\{\pm1\}$, define the restriction $V_{x\to \ell}:=\{h\in V \sep h(x)=\ell\}$.
A sample $(x,y)$ is \emph{realizable} in $V$ if $V_{x\to y}\neq \emptyset$.
We denote the set of all version spaces by  $\mathbb{V}:=\{\Vcal(U)\sep U\subset \Lcal\}$ .

Let $\mathcal{P}_n(\mathcal{L})$ denote the set of subsets of $\Lcal$ of size $n$, that is, 
$$\mathcal{P}_n(\mathcal{L})=\{U\subseteq \Lcal\sep |U|=n\}$$
We also define $\sign:\mathbb{R}\to\{+1,-1\}$ as $$\sign(x)=\begin{cases}
    +1&x\geq 0\\
    -1& x<0
\end{cases}$$
\subsection{Adversarial Injection model}
Here we describe the adversarial injection model of \citep{goel2024adversarialresiliencesequentialprediction}. At the start, an adversary chooses a distribution~$\D$ over~$\X$ and a target $c^\star\in\C$; both remain fixed throughout. The interaction proceeds as follows:

\crefalias{algocf}{protocol}
\SetAlgorithmName{Protocol}{protocol}{protocols}
\begin{algorithm}
    \SetAlgoLined
    \caption{Adversarial Injection Model}
    \label{prot:adversarial_injection}
    \SetKwIF{If}{ElseIf}{Else}{if}{:}{else if}{else}{}
    \SetKwFor{For}{for}{:}{}
    \KwIn{Time horizon $T$}
    
    \For{$t=1, \dots, T$}{
        Adversary chooses a hidden bit $q_t \in \{0,1\}$\;
        \eIf{$q_t=0$}{
            Nature draws $x_t \sim \mathcal{D}$ independently of the past (i.i.d.\ sample)\;
        }{
            Adversary chooses $x_t \in \mathcal{X}$ (injected sample)\;
        }
        Learner observes $x_t$ and outputs $\hat{y}_t \in \{\pm 1, \perp\}$\;
        Learner observes the true label $y_t = c^\star(x_t)$\;
    }
\end{algorithm}
\crefalias{algocf}{algorithm}
\SetAlgorithmName{Algorithm}{algorithm}{Algorithms}
Note that the learner never observes $q_t$, and the adversary may be adaptive: on round~$t$, both its choice of $q_t$ and (if $q_t=1$) the injected sample $x_t$ may depend on the history up to time~$t-1$. Unless otherwise noted, we assume $\D$ is hidden from the learner. We write $S_t:=\{(x_s,y_s)\sep 1\le s\le t\}$ for the set of labeled samples seen up to time~$t$, and $V_t:=\Vcal(S_t)$ for the corresponding version space.\footnote{Duplicates do not affect version spaces, and the order is generally unimportant, so we treat histories as sets.}
Under realizability, $V_t$ is always nonempty.

\paragraph{Objective.}
The learner is penalized for two types of errors: \emph{misclassifications}, where the learner predicts but gets the label wrong, and \emph{incorrect abstentions}, where the learner abstains on an i.i.d.\ round. Crucially, abstaining on injected rounds incurs no cost. Formally:
\begin{align*}
    err_{\mathrm{mis}}
    :=\sum_{t=1}^T \1\bigl[\hat y_t\neq y_t\ \wedge\ \hat y_t\neq \perp\bigr] \qquad\text{and}\qquad
    err_{\mathrm{abs}} :=\sum_{t=1}^T \1\bigl[\hat y_t=\perp\ \wedge\ q_t=0\bigr].
\end{align*}
We study the expected combined error $\E[err_{\mathrm{mis}}+err_{\mathrm{abs}}]$, where the expectation is over the i.i.d.\ draws, any learner randomization, and (when applicable) adversary randomization.
\section{The limits of unknown distributions}\label{sec:lower-bound}

We begin by establishing the inherent cost of distribution-agnostic learning.
In the unknown-$\D$ regime, even very simple classes force $\Omega_d(\sqrt{T})$ total error, ruling out the $\mathcal{O}_d(\log T)$ rates achievable with marginal access and setting the baseline for the upper bounds later in the paper.

\begin{theorem}[A hard VC-1 class]\label{thm:vc1-lb}
There exists a concept class $(\X,\C)$ with $\VCdim(\C)=1$ such that for every horizon $T$, there exists an adversary/distribution choice in the unknown-$\D$ injection model such that for every learner,
\[
\E[err_{\mathrm{mis}}] + \E[err_{\mathrm{abs}}]=\Omega(\sqrt T).
\]
\end{theorem}

\paragraph{Proof Sketch}
We have $\X$ be the nodes of a rooted tree of depth at least $\sqrt{T}$, with each non-leaf node having at least $\sqrt{T}$ children. Our VC-1 concept class is root to leaf paths in the tree, that is, for each root to leaf path, there is a concept that is $+1$ on each node in the path, and $-1$ on all other nodes. See \Cref{fig:tree} for an illustration. 

Our adversary will make it so that each time the learner makes a mistake, they won't learn anything useful for future predictions. The adversary first uniformly at random chooses a node $\theta_B$ at depth $B=\Theta(\sqrt{T})$. Let $\theta_0,\theta_1,\dots \theta_B$ denote the path from the root ($\theta_0$) to $\theta_B$. We split time into $B$ contiguous blocks of length $\sqrt{T}$. For each block $i\in [B]$, let $D_i$ the uniform distribution over $B$ children of $\theta_i$, including $\theta_{i+1}$ as one of the children in the support. The adversary then chooses (uniformly at random) exactly one of the distributions $D_i$ to be the global i.i.d.\ distribution. Then, in each the $i^{th}$ block of time, the adversary either injects samples from $D_i$, or allows i.i.d.\ samples if $D_i$ is the i.i.d.\ block. The choice of which $D_i$ is i.i.d.\ has no affect on the samples themselves, so the learner's perspective is identical no matter which $D_i$ is chosen.

Due to the independence of $q_t$ (whether the sample at time $t$ is i.i.d.) and $x_t$ (the actual sample at time $t$) for our adversary, whenever the learner abstains, their expected abstention error, over the randomness of the adversary, is the probability that the sample is i.i.d., $1/B$. 

We now use the fact that the labels in each block are independent from the labels of the other blocks to analyze each block in isolation. During block $i$, until $\theta_i$ is revealed, from the perspective of the learner, each label $y_t$ is $+1$ with probability $1/\sqrt{T}$ (and $-1$ otherwise). Since the adversary is oblivious, we can use Yao's minimax principle, which shows that when not abstaining, and before $\theta_i$ is revealed, the learner can do no better than predicting $\hat{y}_t=-1$, getting expected misclassification error $\frac{1}{\sqrt{T}}$. A standard coupon counting argument shows that with probability very close to $1/e$, $\theta_i$ is never sampled, allowing us to lower bound the expected misclassification error conditioned on not abstaining with $\frac{1}{e\sqrt{T}}$.

Therefore, at each time $t$, the learner either incurs expected abstention error $\frac{1}{B}=\frac{1}{\Omega(B)}$, or expected misclassification error at least $\frac{1}{e\sqrt{T}}=\frac{1}{\Omega(B)}$, so over $T$ rounds, the expected sum of the two errors is $\Omega(B) = \Omega(\sqrt{T}).$

For the formal derivation, see Appendix \ref{sec:proof of vc1}.

\begin{figure}
    \centering
    \resizebox{0.7\columnwidth}{!}{{
        \usetikzlibrary{arrows.meta, positioning, calc, positioning, quotes}
\renewcommand{\seriesdefault}{\bfdefault}
\boldmath
\begin{tikzpicture}[
    node distance=1.5cm and 8mm,
    every node/.style={circle, draw, minimum size=8mm, line width=2pt},
    label/.style={rectangle, midway, fill=white, draw=none,minimum size=0.1mm, inner sep = 1mm},
    selected/.style={circle, draw, fill=blue!20, line width=2.5pt, minimum size=8mm},
    edge/.style={->, >=Stealth, line width=2pt},
    selected edge/.style={->, >=Stealth, thick, blue!70, line width=2.5pt},
]
\def\layersep{3}
\def\symb{\theta}

\node[selected] (root) at (0,0) {};

\node[draw=none, font=\Large] at (0,5) {Layer 0 (root)};

\node (c1) at (\layersep,3) {};
\node (c2) at (\layersep,1.2) {};
\node[draw=none] (cdots1) at (\layersep,-0.5) {$\vdots$};
\node[selected] (c3) at (\layersep,-2.2) {};
\node[draw=none] (cdots2) at (\layersep,-3.5) {$\vdots$};
\node (ck) at (\layersep,-5) {};

\draw[edge] (root) -- node[label] {$1$} (c1);
\draw[edge] (root) -- node[label] {$2$} (c2);
\draw[selected edge] (root) -- node[label] {$\symb_1$} (c3);
\draw[edge] (root) -- node[label] {$B$} (ck);

\node[draw=none, font=\Large] at (\layersep,5) {Layer 1};

\node (c3c1) at ({2*\layersep},-0.8) {};
\node[draw=none] (c3dots1) at ({2*\layersep},-1.5) {$\vdots$};
\node[selected] (c3c2) at ({2*\layersep},-2.2) {};
\node[draw=none] (c3dots2) at ({2*\layersep},-2.9) {$\vdots$};
\node (c3c3) at ({2*\layersep},-3.6) {};

\draw[edge] (c3) -- node[label] {$1$}(c3c1);
\draw[selected edge] (c3) --  node[label] {$\symb_2$}  (c3c2);
\draw[edge] (c3) --  node[label] {$B$} (c3c3);

\draw[edge, gray!50] (c1) -- node[label] {$1$} ({2*\layersep},3.5);
\draw[edge, gray!50] (c1) --  node[label] {$B$} ({2*\layersep},2.5);
\node[draw=none, gray!50] at ({1.8*\layersep},3) {$\vdots$};
\draw[edge, gray!50] (c2) -- node[label] {$1$} ({2*\layersep},1.7);
\draw[edge, gray!50] (c2) --  node[label] {$B$} ({2*\layersep},0.7);
\node[draw=none, gray!50] at ({1.8*\layersep},1.2) {$\vdots$};
\draw[edge, gray!50] (ck) -- node[label] {$1$} ({2*\layersep},-4.5);
\draw[edge, gray!50] (ck) --  node[label] {$B$}  ({2*\layersep},-5.5);
\node[draw=none, gray!50] at ({1.8*\layersep},-5) {$\vdots$};

\node[draw=none, font=\Large] at ({2*\layersep},5) {Layer 2};

\def\gap{.3}
\node[draw=none, font=\Large] (hdots) at ({(2.8+\gap)*\layersep},0) {$\hdots$};
\node[draw=none, font=\Large] at ({(2.8+\gap)*\layersep},5) {$\hdots$};

\node[selected] (pen2) at ({(4+\gap)*\layersep},-2.8) {};

\draw[selected edge] (c3c2) -- node[label, near start] {$\symb_3$} ({(2.8)*\layersep}, -3.2);
\draw[selected edge] ({(\gap+3.4)*\layersep}, -0.5) -- node[label] {$\symb_{B-1}$}(pen2);

\node[draw=none, font=\Large] at ({(\gap+4)*\layersep},5) {Layer $B-1$};

\node[rectangle] (leaf1) at ({(\gap + 5)*\layersep},-1.3) {};
\node[draw=none] (leafdots1) at ({(\gap + 5)*\layersep},-2.15) {$\vdots$};
\node[selected, rectangle] (leaf2) at ({(\gap + 5)*\layersep},-3) {};
\node[draw=none] (leafdots2) at ({(\gap + 5)*\layersep},-3.85) {$\vdots$};
\node[rectangle] (leafk) at ({(\gap + 5)*\layersep},-4.7) {};

\draw[edge] (pen2) -- node[label] {$1$} (leaf1.west);
\draw[selected edge] (pen2) -- node[label] {$\symb_{B}$} (leaf2);
\draw[edge] (pen2) -- node[label] {$B$} (leafk.west);

\node[draw=none, font=\Large] at ({(\gap + 5)*\layersep},5) {Layer $B$ };

\end{tikzpicture}}}
    \caption{Hard VC-1 class. The edges are labeled to indicate the order of the children of the nodes. The path $\theta$ is highlighted in blue. }
    \label{fig:tree}
\end{figure}
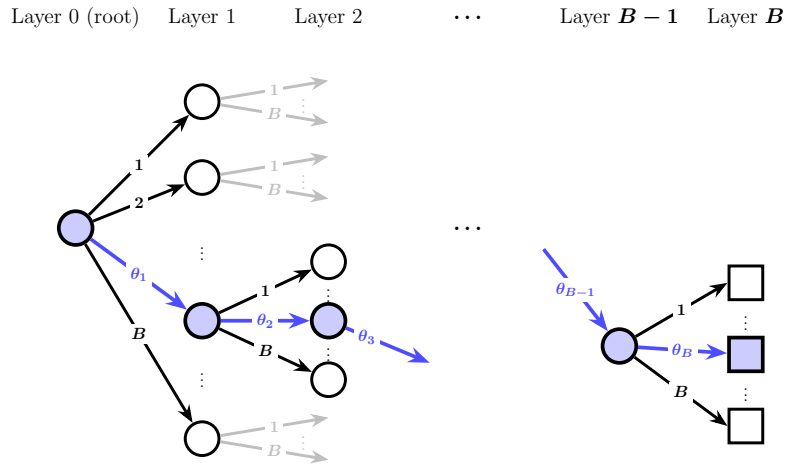

\section{A Potential-Based Framework for Learning}\label{sec:framework}
w
We now present a general potential-based framework for upper bounds in the adversarial injection model. Prior work \citep{goel2024adversarialresiliencesequentialprediction} gives two distribution-agnostic algorithms, one for VC dimension~$1$ classes and one for axis-aligned rectangles, that at first glance look quite different. At a high level, however, both follow the same principle: on a new instance $x_t$, the learner either abstains or makes a prediction that (if wrong) provably forces the version space to shrink in a meaningful way. Thus, every misclassification comes with a certificate of progress, yielding a mistake bound. The charged-abstention bound then follows by showing that i.i.d.\ samples cannot keep presenting rounds on which \emph{neither} label would certify progress by a more involved \emph{attackability} argument.

We make this connection explicit by introducing a family of \emph{leave-$k$-out} potentials and a generic learner parameterized by a score function $f$. The score $f(U;V)$ should be thought of as measuring how much \emph{uncertainty} remains about a small labeled set $U$ when the current feasible hypotheses are $V$. Aggregating this score over many subsets of the observed history is what makes the resulting potential robust to clean-label injections. This perspective was suggested as a future direction in \citep{goel2024adversarialresiliencesequentialprediction} via a particular leave-$k$-out shattering estimator; we allow a general $f$ and isolate simple sufficient conditions for both mistake and abstention guarantees.

\begin{definition}[Leave-$k$-out potential]
Fix an integer $k\ge 1$ and a function $f:\mathcal{P}_k(\mathcal{L}) \times \mathbb{V}\to \mathbb{R}$. The \emph{leave-$k$-out potential} of a finite labeled set $S\subseteq \Lcal$ is
\[
\rho_{f}(S):=\sum_{\substack{U\subseteq S\\|U|=k}} f\bigl(U;\Vcal(S\setminus U)\bigr).
\]
\end{definition}
\begin{algorithm}[t]
\caption{Potential Learner}\label{alg:potential-learner}
\DontPrintSemicolon
\SetKwIF{If}{ElseIf}{Else}{if}{:}{else if}{else}{}
\SetKwFor{For}{for}{:}{}
\KwIn{Function $f$; threshold $\alpha>0$}
$S_0\gets\emptyset$\;
\For{$t=1,2,\dots,T$}{
    Receive $x_t$\;
    \uIf{$\Vcal(S_{t-1})_{x_t\to +1}=\emptyset$}{$\hat{y}_t\gets -1$}
    \uElseIf{$\Vcal(S_{t-1})_{x_t\to -1}=\emptyset$}{$\hat{y}_t\gets +1$}
    \Else{
        \For{$b\in\{\pm1\}$}{
            $\Delta_{b}\gets \rho_f(S_{t-1})-\rho_f(S_{t-1}\cup\{(x_t,b)\})$\;
        }
        \uIf{$\max\{\Delta_{+1},\Delta_{-1}\}\ge \alpha$}{
            $b^\star\gets \arg\max_{b\in\{\pm1\}}\Delta_b$\;
            $\hat{y}_t\gets -b^\star$\;
        }
        \Else{$\hat{y}_t\gets \perp$}
    }
    Observe $y_t$ and update $S_t\gets S_{t-1}\cup\{(x_t,y_t)\}$\;
}
\end{algorithm}

Our learner (\Cref{alg:potential-learner}) predicts a label $\hat{y}_t\in\{-1,+1\}$ only when it can \emph{certify} that if this label is wrong, then the potential must drop by at least $\alpha$. Equivalently, if $b^\star$ is the label that would maximize the one-step drop, the learner predicts $-b^\star$ so that a mistake (i.e.\ $y_t=b^\star$) forces a drop of at least $\alpha$.

The remainder of the section shows that this learner achieves sublinear combined error whenever the score function $f$ satisfies a few properties.

\begin{definition}[\good function]
    Fix a concept class $\C$. We call a function $f:\mathcal{P}_k(\mathcal{L})\times \mathbb{V}\to [0,c]$ \emph{\good}if it satisfies:
\begin{enumerate}
    \item \emph{(Monotonicity)} For any $U\subseteq \Lcal,|U|=k,$ and version space $V$ consistent with the labels of $U$, and for any $(x,y)\in \Lcal$ realizable with $U$ in $V$, $$f(U;V)\geq f(U;V_{x\to y})$$
    \item \emph{($m$-Robustness)} For any $h\in \C$, for any realizable $M\subseteq \Lcal,|M|=m,$ labeled by $h$, there exists $U\subseteq M,|U|=k,$ and a witness $(x,y)\in M\setminus U$, such that for every realizable extension $A\supseteq M\setminus (U\cup\{(x,y)\})$ labeled by $h$, letting $V = \Vcal(A)$, either
    $$f(U;V)-f(U;V_{x\to -y}) \geq 1$$
    or $U$ is unrealizable in $V_{x\to -y}$
\end{enumerate}
\end{definition}

\begin{theorem}\label{thm:upper bound}
    \Cref{alg:potential-learner}, when run in the adversarial injection setting with an \good function $f$ and threshold $\alpha>0$, satisfies:
    \begin{align*}
        \errmis \leq \frac{c}{k! \alpha}T^k\qquad \text{and} \qquad \E[\errabs]\leq em\bigl(\alpha + cT^{k-1}\bigr)^{1/k}\log T
    \end{align*}
    We can asymptotically bound the combined error by choosing the optimal $\alpha$, giving
    \begin{align*}
        \mathbb{E}[\errmis+\errabs]=\tilde{\mathcal{O}}\left( \frac{c^{\frac{1}{k+1}}}{k}\left(mT\right)^{1-\frac{1}{k+1}}\right)
    \end{align*}
\end{theorem}

\subsection{Proof Sketch for \Cref{thm:upper bound}}

The proof of \Cref{thm:upper bound} works by separately bounding each type of error as a function of $\alpha$, and then choosing $\alpha$ to balance them. See full proof in Appendix~\ref{sec:proof of framework}.

\paragraph{Bounding $\errmis$.}
We will call $\rho_f(S_t)-\rho_f(S_{t-1})$ the one-step \emph{increase}.
By the learner's rule, every mistake forces a certified drop, i.e.\ on any round with $\hat y_t=-y_t$ we have $\rho_f(S_{t-1})-\rho_f(S_{t})\ge \alpha$,

It remains to upper bound how much the potential can increase on an arbitrary round. Let $z_t:=(x_t,y_t)$ denote the new labeled sample.
When we add $z_t$ to form $S_t=S_{t-1}\cup\{z_t\}$, the potential $\rho_f(S_t)$ consists of (i) terms indexed by $k$-subsets $U$ that do \emph{not} contain $z_t$, and (ii) new terms indexed by $k$-subsets that \emph{do} contain $z_t$.

For type-(i) subsets, $S_t\setminus U$ contains all of $S_{t-1}\setminus U$ plus $z_t$, so the corresponding version space only shrinks; by monotonicity of $f$, these terms can only decrease compared to $\rho_f(S_{t-1})$.
For type-(ii) subsets, we simply use boundedness $f\le c$ and count how many such subsets exist, namely $\binom{t-1}{k-1}$.
Thus the net increase satisfies
\[
\rho_f(S_t)-\rho_f(S_{t-1}) \;\le\; c\binom{t-1}{k-1}\qquad\text{for all }t.
\]

Finally, sum increments over time:
\[
\rho_f(S_T)=\sum_{t=1}^T \rho_f(S_t)-\rho_f(S_{t-1}) \;\ge\; 0.
\]
Splitting the sum into mistake rounds (each contributing at most $-\alpha$) and all rounds (each contributing at most $c\binom{t-1}{k-1}$) yields
\[
0 \;\le\; \rho_f(S_T) \;\le\; -\alpha\,\errmis \;+\; c\sum_{t=1}^T \binom{t-1}{k-1}
\;=\; -\alpha\,\errmis \;+\; c\binom{T}{k}.
\]
Rearranging gives us $\errmis \le \frac{c}{\alpha}\binom{T}{k}$.

\paragraph{Bounding $\E[\errabs]$.} 

The abstention analysis is the main use of $m$-robustness. Informally, robustness says that in any realizable set $M$ of size $m$, there is a \emph{witness} point $(x,y)\in M$ and a \emph{witness} $k$-subset $U\subseteq M$ such that, no matter what additional (realizable) samples are later added outside of $U\cup\{(x,y)\}$, flipping the label of $x$ forces $f(U;\cdot)$ to drop by at least~$1$.

We will show that for any large enough realizable set $H$, the learner will predict on all but a few points once it has seen the rest of $H$ (even allowing further truthful injections). Fix a realizable labeled set $H$ and let $n:=|H|$. Suppose (for contradiction) that for every $(x,y)\in H$ there exists a truthfully labeled extension $A_{x,y}\supseteq H\setminus\{(x,y)\}$ with $|A_{x,y}|\le T$ such that the learner would abstain on $x$ given history $A_{x,y}$. Applying $m$-robustness to all $M\subseteq H,|M|=m
,$ and a pigeonhole argument yields a point $(x,y)\in H$ with $\Omega\left(n^k/m^k\right)$ distinct witness subsets $U\subseteq H\setminus\{(x,y)\}$ of size $k$, each forcing a unit drop in the corresponding leave-$k$-out term when we flip the label of $x$.

Summing these contributions lower bounds the counterfactual potential drop $\rho_f(A_{x,y})-\rho_f(A_{x,y}\cup\{(x,-y)\})$ to be at least $\Omega(n^k/m^{k})-\mathcal{O}(cT^{k-1})$. Since abstention means both potential drops (for $y$ and $-y$) are $<\alpha$, this quantity must be $<\alpha$. Rearranging gives $n=\mathcal{O}\!\left(m(\alpha+cT^{k-1})^{1/k}\right)$. Concretely, among any set of $i$ i.i.d.\ samples, at most $n$ of them can be points on which the learner could ever abstain once the history contains the other $i-1$ i.i.d.\ samples (allowing additional truthful injections up to total size $T$).

By exchangeability, this implies that the probability of abstaining on the $i$th i.i.d.\ draw is at most this quantity divided by $i$.
Bounding the sum over $i\le T$ yields the stated bound on $\E[\errabs]$.

\paragraph{Choosing $\alpha$.}
In the regimes of interest we have $\alpha \gtrsim cT^{k-1}$, so up to constants the bounds behave like $\errmis=\mathcal{O}(\frac{c}{k!\alpha} T^k)$ and $\E[\errabs]=\mathcal{O}(m\alpha^{1/k}\log T)$.
We choose $\alpha$ to balance these, giving combined error $\mathcal{O}\left(\frac{1}{k} c^{\frac{1}{k+1}} (mT \log T)^{1-\frac{1}{k+1}}\right).$

\subsection{Recovering Axis-aligned Rectangles and VC-1 Results}\label{subsec:prior art}

We first show that the two unknown-$\D$ upper bounds in \cite{goel2024adversarialresiliencesequentialprediction} are special cases of \Cref{alg:potential-learner}. Their analyses implicitly verify that the corresponding score functions are \good (for the appropriate choices of parameters), and thus follow from \Cref{thm:upper bound}.

\paragraph{Axis-aligned rectangles.}
For rectangles containing the origin, the learner's uncertainty is essentially \emph{coordinate-wise}: each dimension $i$ corresponds to whether the boundary in that coordinate has already been pinned down by the history.
\begin{definition}[Axis-aligned rectangles containing the origin in $\mathbb{R}^d_{\geq0}$]\label{def:axis aligned rectangles}
    We define axis-aligned rectangles in $d$ dimensions to be the concept class $\C=\{h_v\sep v\in \mathbb{R}_{\geq0}^d\}$ where $h_v(x)=-1 + 2\prod_{i=1}^d \mathbf{1}[x_i\leq v_i]$ is $+1$ if $x\in\mathbb{R}_{\geq0}^d$ is at most $v$ in every coordinate, and $-1$ otherwise.
\end{definition}
\begin{definition}[Rectangle score function]\label{def:rect-score}
Let $\C$ be axis-aligned rectangles containing the origin in $\X=\mathbb{R}_{\geq0}^d$.\footnote{It's very similar to show the same result for axis-aligned rectangles in $\mathbb{R}^d$, just increasing $m$ from $d$ to $2d$. We chose not to for the purpose of simplifying notation.}
Define $f_{\mathrm{rect}}$ on singletons $U=\{(x,y)\}$ by
\[
    f_{\mathrm{rect}}(U;V):=\left|\left\{i\in[d] \;\middle|\; (\delta_i(x),-1)\text{ is realizable in }V_{x\to y}\right\}\right|,
\]
where $\delta_i(x)\in\mathbb{R}_{\ge 0}^d$ is the vector with $(\delta_i(x))_i=x_i$ and $(\delta_i(x))_j=0$ for $j\neq i$.
\end{definition}

\begin{lemma}\label{lem:axis aligned rectangles}
    Fix $\C,\X$ to be axis-aligned rectangles containing the origin in $\mathbb{R}^d_{\geq 0}$. Let $f_{\mathrm{rect}}$ be the score function from \Cref{def:rect-score}. Then $f_{\mathrm{rect}}$ is a \good[2d+1][d][1] function, and \Cref{alg:potential-learner} with $f_{\mathrm{rect}}$ and correct choice of $\alpha$ achieves combined error $\tilde{\mathcal{O}}\left(d \sqrt{T}\right)$.

\end{lemma}

\paragraph{VC dimension 1.} We now recover the VC-$1$ upper bound.
We use a standard structural characterization \citep{bendavid20152notesclassesvapnikchervonenkis}: $\VCdim(\C)\le 1$ iff there exist a partial order $\prec$ over $\X$ such that every initial segment $I(x)=\{x'\in\X\sep x'\preceq x\}$ is totally ordered (a \emph{tree order}), and a reference labeling $r:\X\to\{\pm1\}$ such that, for every relabeled hypothesis $h_r(x):=-h(x)\,r(x)$, the set $h_r^{-1}(+1)$ is an initial segment under $\prec$. Thus, after relabeling, it suffices to analyze the class of initial segments under a tree order.

\begin{definition}[VC-$1$ score function for initial segments]\label{def:vc1-score} Let $\prec$ be a tree order on $\X$ and let $\C$ be the class of initial segments under $\prec$. For a version space $V$, let $P(V)$ be the set of points that are \emph{forced} to be $+1$ under $V$ as positive\footnote{$P(V):=\{x'\sep x'\in\X,\; V_{x'\to -1}=\emptyset\}$}. Define $f_{\mathrm{seg}}$ on singletons $U=\{(x,y)\}$ by
\[
f_{\mathrm{seg}}(\{(x,y)\};V)= \mathbf{1}\!\left[P(V)\subseteq I(x)\right]
\]
Where $I(x)$ is the initial segment of $x$ under $\preceq$.
\end{definition}
    
\begin{lemma}\label{lem: tree order}
For a domain $\X$ with a tree order $\prec$, let $\C$ be the class of initial segments under $\prec$. Let $f_{\mathrm{seg}}$ be the score function from \Cref{def:vc1-score}. Then $f_{\mathrm{seg}}$ is a \good[3][1][1] function and \Cref{alg:potential-learner} with $f_{\mathrm{seg}}$ and the right choice of $\alpha$ achieves combined error $\tilde{\mathcal{O}}\left(\sqrt{T}\right)$.
\end{lemma}
Due to the aforementioned characterization of any VC $1$ concept class as initial segments on a tree ordering, we obtain the more general result as a corollary of \Cref{lem: tree order}.
\begin{corollary}\label{lem: vc1}
    Let $\C$ be a VC $1$ concept class over domain $\X$. There exists a function $f_{\mathrm{vc1}}$ that is \good[3][1][1], and \Cref{alg:potential-learner} with $f_{\mathrm{vc1}}$ and the right choice of $\alpha$ achieves combined error $\tilde{\mathcal{O}}\left(\sqrt{T}\right)$.
\end{corollary}

\section{New Upper Bounds: Inference Dimension and Certificate Dimension}\label{sec:upper_bounds}
Inference dimension is a combinatorial parameter from active learning with comparison queries \citep{kane2017activeclassificationcomparisonqueries}. We use it as a sufficient condition for instantiating our potential framework.

\begin{definition}[Inference dimension {\citep{kane2017activeclassificationcomparisonqueries}}]
Let $\C$ be a concept class. Assume $\C$ admits a real-valued representation: for each $h\in\C$ there exists $f_h:\X\to\mathbb{R}$ such that $h(x)=\sign(f_h(x))$ for all $x$. Given a finite set $S=\{x_1,\dots,x_s\}\subseteq \X$ and $h\in\C$, consider the comparison-query transcript
\[
Q_h(S):=\Bigl(\bigl(\1[f_h(x_i)\ge 0]\bigr)_{i\in[s]},\; \bigl(\1[f_h(x_i)\ge f_h(x_j)]\bigr)_{i,j\in[s]}\Bigr).
\]
We say that $S\setminus\{x\}$ \emph{infers} $x$ under $h$ if for every $h'\in\C$ with $Q_{h'}(S\setminus\{x\})=Q_h(S\setminus\{x\})$, we have $h'(x)=h(x)$. The \emph{inference dimension} of $(\C,\X)$ is the smallest $k$ such that for every $S\subseteq \X,|S|=k$ and every $h\in\C$ there exists $x\in S$ such that $S\setminus\{x\}$ infers $x$ under $h$.
\end{definition}
\paragraph{From inference dimension to our framework.} For $|S|=k-1$, the transcript $Q_h(S)$ takes only finitely many values; in fact there are at most $2^{k-1}k!$ possibilities. Encoding the transcript as a finite-valued certificate yields a \good[k][\ensuremath{2^{k-1}k!}][k-1] score function.
\begin{lemma}\label{lem:inference dimension}
If $(\C,\X)$ has inference dimension $k$, then there exists a \good[k][\ensuremath{2^{k-1}k!}][k-1] function $f$ such that \Cref{alg:potential-learner} (with an appropriate choice of $\alpha$) achieves combined error bounded by $\tilde{\mathcal{O}}\left(kT^{1-\frac{1}{k}}\right)$.
\end{lemma}
For example, halfspaces in $\mathbb{R}^2$ have inference dimension $5$ \citep{kane2017activeclassificationcomparisonqueries}, yielding a $\tilde{\mathcal{O}}(T^{4/5})$ bound via \Cref{lem:inference dimension}.
We can also use inference dimension to show that halfspaces with finite bit precision, or with a margin, are learnable.

\paragraph{Certificate Dimension.} Inference dimension is tailored to comparison queries and real-valued representations. Our framework only needs a finite-valued notion of \emph{certificates}, which motivates the following relaxation.

\begin{definition}[Certificate dimension]
We say that $(\C,\X)$ has \emph{certificate dimension $k$ with parameters $(m,n)$} if there exist integers $m,n\ge 1$ and a function $\sigma:\C\times\mathcal{P}_{k-1}(\mathcal{L})\to\{0,1,\dots,n-1\}$ such that the following holds. For every $h\in\C$ and every realizable $S\subseteq \Lcal, |S|=m$ labeled by $h$, there exist $U\subseteq S,|U|=k$ and $(x,y)\in U$ such that for every $h'\in \Vcal(S\setminus\{(x,y)\})$,
\[
\sigma\left(h',U\setminus\{(x,y)\}\right) =\sigma\left(h,U\setminus\{(x,y)\}\right) \Longrightarrow h'(x)=h(x)=y.
\]
\end{definition}

Compared to inference dimension, certificate dimension introduces a tradeoff between the \emph{certificate alphabet} size $n$ and the \emph{window} size $m$: we only require that \emph{some} $k$-tuple inside any realizable set of size $m$ contains a certifying point. Allowing $m>k$ is what lets us use smaller certificates rather than encoding a full comparison-query transcript on a size-$(k-1)$ set.
 \begin{lemma}\label{lem:certificate dimension}
 If $(\C,\X)$ has certificate dimension $k$ with parameters $(m,n)$, then there exists a \good[m][n][k-1] function $f$ such that \Cref{alg:potential-learner} with an appropriate choice of $\alpha$ achieves
\[
\E[\errmis+\errabs] \leq \tilde{\mathcal{O}}\left( \frac{1}{k}n^{\frac{1}{k}} \left(mT\right)^{1-\frac{1}{k}}\right).
\]
\end{lemma}

\paragraph{Proof sketch.} We take $f(U;V)$ to be the number of distinct certificate values $\sigma(h,U)$ realizable by hypotheses $h\in V$.
This is bounded and monotone, and the certificate-dimension condition ensures that in every realizable set of size $m$ there is a point whose mislabeling eliminates the true certificate value, yielding robustness.

\paragraph{Halfspaces in $\mathbb{R}^2$.} While inference dimension already implies learnability (giving a $\tilde{\mathcal{O}}(T^{4/5})$ bound via \Cref{lem:inference dimension}), certificate dimension yields a sharper exponent\footnote{For technical reasons, we bound the certificate dimension of halfspaces in $\mathbb{R}^2$ only when the learner starts out with a sample of each label. This can be easily achieved by the learner with no abstentions and only two mistakes. See the proof in \Cref{sec:proof of upper bounds} for more details.} The proof relies on the combined error having a good dependence on the window size, and the certificate being a function of just two points.
\begin{theorem}\label{thm:halfplanes}
    Let $\X=\mathbb{R}^2$ and $\C=\{h(x)=\sign\left(w^\top x + b\right)\sep w\in \mathbb{R}^2, b\in \mathbb{R}\}$ be the class of two dimensional halfspaces. $\C$ can be learned in the adversarial injections model, with
    \begin{align*}
        \mathbb{E}[\errmis+\errabs] \leq \tilde{\mathcal{O}}\left(T^{2/3}\right)
    \end{align*}
\end{theorem}
The certificate used for half-spaces is $$\sigma(h,\{(x_1,y_1),(x_2,y_2)\})=\begin{cases}
    0 & w_h^\top x_1> w_h^\top x_2\\
    1 & w_h^\top x_1= w_h^\top x_2\\
    2 & w_h^\top x_1< w_h^\top x_2
\end{cases}$$ where $x_1,x_2$ are sorted lexicographically by their coordinates, and $w_h$ is defined so that $h(x)=\sign\left(w_h^\top x - b\right)$ for some $b\in \mathbb{R}$.

\begin{remark}
    The certificate dimension of axis aligned rectangles is $2$.
\end{remark}
\section{Discussion}
This paper extends our understanding of sequential prediction with abstention in the adversarial injection model. We proved a sharp separation between the known-$\D$ and unknown-$\D$ regimes: even for a VC-$1$ class, any learner without access to $\D$ must incur $\Omega(\sqrt{T})$ expected combined error, whereas oracle access to $\D$ enables polylogarithmic rates. On the algorithmic side, we introduced a potential-based framework that unifies the distribution-agnostic upper bounds from prior work and yields new ones, including learnability for every class with finite inference dimension and the first bound for halfspaces in $\mathbb{R}^2$. Despite this progress, there are many unresolved problems.

\paragraph{Stronger lower bounds via stronger adversaries.}
Our separation lower bound uses an oblivious adversary.
An interesting direction is to understand whether adaptive injection strategies can force larger error for other natural classes (beyond those where $\tilde{\mathcal{O}}(\sqrt{T})$ is already known to be tight), or whether the oblivious construction is essentially worst-case for the combined objective.

\paragraph{Upper bounds beyond exchangeability-based analyses.}
The abstention bounds in both this work and \citet{goel2024adversarialresiliencesequentialprediction} ultimately rely on exchangeability of the i.i.d.\ samples.
This style of argument is robust to very strong injection strategies (including ones that are tailored to individual i.i.d.\ points).
It is plausible that for some classes, restricting the adversary (or strengthening the learner's information) could yield faster rates, but doing so likely requires new techniques for controlling charged abstentions beyond the current exchangeability-based approach.

\paragraph{Higher-dimensional halfspaces.}
While we show halfspaces in $\mathbb{R}^2$ are learnable with sublinear combined error, the status of halfspaces in $\mathbb{R}^d$ for $d\ge 3$ remains open.
In particular, inference dimension is infinite in higher dimensions, but it is unclear whether this reflects a real barrier in the adversarial injection model with abstention, or merely a limitation of inference-dimension-based instantiations. Even a rate of $\mathcal{O}(T^{1 - 1/(d+1)})$ would be interesting.

\paragraph{A characterizing combinatorial dimension.}
A central open problem is to find a combinatorial parameter that characterizes learnability (and rates) in this model, analogous to VC dimension in the stochastic setting and Littlestone dimension in the fully adversarial setting. Our results suggest that the landscape is more delicate: VC dimension alone does not determine whether stochastic-like rates are achievable without access to $\D$, and new dimensions (such as inference or certificate dimension) capture additional structure.

\paragraph{Intermediate access models for $\D$.}
Our separation highlights the role of distribution access, but there are many intermediate regimes between full oracle access and fully hidden $\D$. For example, if the learner only has sample access to $\D$, one can trade off the number of auxiliary samples against abstention and misclassification error. Understanding the sample complexity needed to approximate the known-$\D$ guarantees of \citet{goel2024adversarialresiliencesequentialprediction} (and whether polynomially many samples suffice for broad VC classes) is an appealing direction.

\section*{Acknowledgments}
The authors gratefully acknowledge support from NSF Award CCF:AF:2504016 and a Schmidt Sciences AI2050 Early Career Fellowship.
\bibliographystyle{plainnat}
\bibliography{bib}

\newpage
\appendix
\section{Proof of \Cref{thm:vc1-lb}}\label{sec:proof of vc1}
\begin{proof}
Consider a rooted tree of countably infinite depth, where each node has countably infinite children\footnote{We have the tree be infinitely large so that the same $\C,\X$ can be used regardless of $T$.}. Let $\X$ be the nodes of the tree. For each node $x\in\X$, define $f_x:\mathcal{X}\to \{-1,+1\}$ by $f_x(v)=+1$ iff $v$ is a prefix of $x$, and $-1$ otherwise\footnote{That is, $f_x(v)=+1$ if and only if $v$ is on the path from the root to $x$.}. Let $\C:=\{f_x\sep x\in \mathcal{X}\}$. We claim that $\VCdim(\C)=1$.

First, $\VCdim(\C)\ge 1$ since for any non-root node $v$ there exist leaves $\theta$ that extend $v$ (so $f_\theta(v)=+1$) and leaves $\theta'$ that diverge before reaching $v$ (so $f_{\theta'}(v)=-1$).
On the other hand, no pair of samples can be shattered: for any two distinct nodes $u,v$, either they are incomparable (in which case $f_\theta(u)=f_\theta(v)=+1$ is impossible, since all samples labeled $+1$ are part of a chain, and are comparable to each other) or one is an ancestor of the other (in which case the labeling the ancestor $-1$ and the descendant $+1$ is impossible).
Thus $\VCdim(\C)\le 1$.

Let $B:=\lfloor \sqrt{T}\rfloor$ and set $T_0:=B^2\le T$.
We analyze the first $T_0$ rounds; if $T>T_0$, the adversary can set the remaining rounds to be injected duplicates, which cannot decrease the learner's expected total error, so the lower bound extends to horizon~$T$.

Then, choose a subtree $\mathcal{T}$ of depth $B+1$ rooted at the root of the tree, where each non-leaf node in the subtree has exactly $B$ children in $\mathcal{T}$. Then, we can order each non-leaf node's children $1$ through $B$.

Partition the first $T_0$ rounds into $B$ consecutive blocks of length $B$, indexed by $i\in[B]$. 
Define $\theta_0$ as the root, and choose secret $\theta_i\sim \mathrm{Unif}([B])$ for $i\in [B]$, and index $r\sim\mathrm{Unif}([B])$.

$\theta$ describes a (uniformly random) path from the root to a leaf of $\mathcal{T}$, where $\theta_i$ is the $i^{th}$ node along this path. 

For each block $i$, let $D_i$ be the uniform distribution over the $B$ children of $\theta_{i-1}$ in $\mathcal{T}$.
Under $f_\theta$, exactly one support sample of $D_i$ has label $+1$ and the remaining $B-1$ have label $-1$.
In block $i$, present $B$ independent draws from $D_i$.
Only block $r$ will be i.i.d.\ ($q_t=0$ there) and all other blocks will be injected ($q_t=1$ there).
We set the \emph{global} i.i.d.\ distribution to be $\D:=D_r$.
Since blocks $i\neq r$ are injected, the adversary is free to present samples distributed as $D_i$ there; the i.i.d.\ distribution used on clean rounds is fixed as $\D:=D_r$ throughout.

For each block $i$, let $M_i$ and $A_i$ denote the numbers of mistakes and abstentions made in that block.
Over the first $T_0$ rounds we have $err_{\mathrm{mis}}=\sum_{i=1}^B M_i$ and $err_{\mathrm{abs}}=A_r$.
Moreover, conditioned on $\theta$, the labeled transcript in each block is i.i.d.\ from $D_i$ with labels given by $f_\theta$, regardless of $r$.
Thus $r$ is independent of the learner's transcript and
\[
\E[err_{\mathrm{abs}}]=\E[A_r]=\frac{1}{B}\sum_{i=1}^B \E[A_i].
\]
Therefore
\begin{equation}\label{eq:block-decomp}
\E[err_{\mathrm{mis}}]+\E[err_{\mathrm{abs}}]
=\sum_{i=1}^B \E\!\left[M_i+\frac{A_i}{B}\right].
\end{equation}
Since the adversary strategy above is oblivious (it presents samples independently of the learner's actions), by Yao's minimax principle it suffices to analyze deterministic learners.

Fix a block $i$ and focus on the learner's behavior during its $B$ rounds.

Let $m$ be the number of \emph{distinct} support samples among the $B$ draws in this block, ordered by first appearance as $z_1,\dots,z_m$.
Let $J$ be the index of the first appearance of $\theta_i$, the unique $+1$ support sample among $\{z_1,\dots,z_m\}$ (or $J=m+1$ if it never appears).

For each $j\in[m]$, let $t_j$ be the first time within the block that $z_j$ appears.
On the event $\{J\ge j\}$, the learner has not yet seen a $+1$ label by time~$t_j$.
Conditioned on $m$, symmetry implies $\Prb(J=j\mid m)=\frac{1}{B}$ for $j\in[m]$ and $\Prb(J=m+1\mid m)=1-\frac{m}{B}$, hence
\[
    \Prb(J\ge j\mid m)=1-\frac{j-1}{B}=\frac{B-j+1}{B}.
\]
Therefore, for $j\in[m]$,
\[
    \Prb\!\bigl[f_\theta(z_j)=+1 \,\bigm|\, J\ge j,\, m\bigr]
    =\frac{\Prb(J=j\mid m)}{\Prb(J\ge j\mid m)}
    =\frac{1}{B-j+1}.
\]
In particular, for $j\le m-1$ we have $j\le B-1$ and hence $\Prb[f_\theta(z_j)=+1\mid J\ge j,m]\in[1/B,1/2]$.
Thus on round $t_j$ (for $j\le m-1$), regardless of the learner's action, the conditional expected contribution to $M_i+A_i/B$ is at least $1/B$:
abstaining contributes $1/B$, and any prediction has mistake probability at least $1/B$.
On the final new sample $j=m$, the same bound holds unless $m=B$ and $J\ge m$, in which case $\Prb[f_\theta(z_m)=+1\mid J\ge m,m]=1$ and a deterministic learner could incur zero loss.
This exceptional case is exactly $\{J=m\}$ when $m=B$; it occurs with conditional probability $1/B$ and can reduce the per-round loss by at most $1/B$, hence the $1/B^2$ correction.
We account for this by subtracting an additive $1/B^2$ below.

Therefore,
\[
\E\!\left[M_i+\frac{A_i}{B}\,\middle|\, m\right]
\ge
\sum_{j=1}^{m} \Prb(J\ge j\mid m)\cdot \frac{1}{B}\;-\;\frac{1}{B^2}.
\]
Since $m\le B$,
\[
\E\!\left[M_i+\frac{A_i}{B}\,\middle|\, m\right]
\ge
\sum_{j=1}^{m} \frac{B-j+1}{B^2}\;-\;\frac{1}{B^2}
\ge
\sum_{j=1}^{m} \frac{m-j+1}{B^2}\;-\;\frac{1}{B^2}
=
\frac{m(m+1)}{2B^2}\;-\;\frac{1}{B^2}.
\]
Taking expectations over $m$ and using Jensen's inequality (since $m\mapsto m(m+1)$ is convex) gives
\[
\E\!\left[M_i+\frac{A_i}{B}\right]
\ge
\frac{\E[m(m+1)]}{2B^2}\;-\;\frac{1}{B^2}
\ge
\frac{\E[m]^2}{2B^2}\;-\;\frac{1}{B^2}.
\]
For $B$ draws from a uniform distribution on $B$ samples,
\[
\E[m]=B\Bigl(1-(1-\tfrac{1}{B})^B\Bigr)\ge B(1-e^{-1}),
\]
and hence $\E[M_i+A_i/B]\ge (1-e^{-1})^2/2 - 1/B^2$ uniformly over~$i$.
Plugging into \eqref{eq:block-decomp} and summing over $B$ blocks yields
$\E[err_{\mathrm{mis}}+err_{\mathrm{abs}}]=\Omega(B)=\Omega(\sqrt{T})$.
\end{proof}
\section{Proof of \Cref{thm:upper bound}}\label{sec:proof of framework}
The proof will rely on two lemmas, one bounding the misclassification error, and the other bounding the abstention error.

We will start with misclassification error. 
\begin{lemma}\label{lem:misclassification}
    Fix an \good function $f$ and run \Cref{alg:potential-learner} with threshold $\alpha>0$ for $T$ rounds in the injection model.
    \begin{align*}
        \errmis \leq \frac{c}{k!\alpha}T^k
    \end{align*}
\end{lemma}
Interestingly, the proof bounds the number of misclassifications, not just the expected number, and
doesn't rely on any properties of $\C$.

\paragraph{Intuition:}When the learner misclassifies $x_t$, the potential, $\rho_f(S_t)$ drops by $\alpha$ compared with $\rho_f(S_{t-1})$. Then, using the properties of $f$, we can show that $\rho_f(S_t)$ can only drop so much over the course of $T$ steps. Combining these yields the desired bound.
\begin{proof}
    
    Define the increment $\Gamma_t := \rho_f(S_t)-\rho_f(S_{t-1})$. We first upper bound $\Gamma_t$ for an arbitrary round $t$.
    Let $z_t:=(x_t,y_t)$ be the new labeled example, so $S_t = S_{t-1}\cup\{z_t\}$.
    Decompose
    \begin{align*}
    \rho_f(S_t)
    &= \sum_{\substack{U\subseteq S_t\\|U|=k}} f\!\left(U;\Vcal(S_t\setminus U)\right) \\
    &= \sum_{\substack{U\subseteq S_t\\|U|=k\\ z_t\notin U}} f\!\left(U;\Vcal(S_t\setminus U)\right)
    \;+\; \sum_{\substack{U\subseteq S_t\\|U|=k\\ z_t\in U}} f\!\left(U;\Vcal(S_t\setminus U)\right).
    \end{align*}
    For $U$ with $z_t\notin U$, we have $U\subseteq S_{t-1}$ and
    $S_t\setminus U = (S_{t-1}\setminus U)\cup\{z_t\}$, hence
    $\Vcal(S_t\setminus U)\subseteq \Vcal(S_{t-1}\setminus U)$ and by monotonicity of $f$,
    \[
    f\!\left(U;\Vcal(S_t\setminus U)\right)\le f\!\left(U;\Vcal(S_{t-1}\setminus U)\right).
    \]
    Thus these terms contribute at most $\rho_f(S_{t-1})$ in total.
    For the remaining terms (where $z_t\in U$), we use boundedness $f(\cdot)\le c$ and count the number of $k$-subsets containing $z_t$,
    which is $\binom{t-1}{k-1}$.
    Therefore,
    \[
    \rho_f(S_t)\le \rho_f(S_{t-1}) + c\binom{t-1}{k-1},
    \qquad\text{i.e.,}\qquad
    \Gamma_t \le c\binom{t-1}{k-1}.
    \]
    
    Now sum the increments, because $\rho_f(S_0)=0$:
    \[
    \rho_f(S_T) = \sum_{t=1}^T \Gamma_t.
    \]
    Split into mistake rounds $M:=\{t:\hat y_t=-y_t\}$ and non-mistake rounds:
    on $t\in M$ we have $\Gamma_t\le -\alpha$, while for all $t$ we have $\Gamma_t\le c\binom{t-1}{k-1}$. Notice that the learner only makes mistakes when $-\Gamma_t=\Delta_{y_t}\geq \alpha$, using this fact, and $\rho_f(S_T)\ge 0$,
    \begin{align*}
    0 \le \rho_f(S_T)
    &= \sum_{t\in M}\Gamma_t + \sum_{t\notin M}\Gamma_t \\
    &\le (-\alpha)\,|M| + \sum_{t\notin M} c\binom{t-1}{k-1} \\
    &\le (-\alpha)\,\errmis + c\sum_{t=1}^T \binom{t-1}{k-1}
    = (-\alpha)\,\errmis + c\binom{T}{k}.
    \end{align*}
    Rearranging gives
    \[
    \errmis \le \frac{c}{\alpha}\binom{T}{k} \le \frac{c}{\alpha k!}\,T^k.
    \]

\end{proof}
To prove \Cref{lem:abstention}, we use a similar style of argument as used in \cite{goel2024adversarialresiliencesequentialprediction}, showing that for any large enough set of samples, the learner will never abstain on most of the samples if given the rest as history, no matter what other samples are injected by the adversary. The slight relaxation we make is to limit the number of injected samples such that the total number of samples is at most $T$.

To begin with, we define \emph{attackable samples} very similarly to \cite{goel2024adversarialresiliencesequentialprediction}:
\begin{definition}[Attackable Sample]
    A sample $(x,y)$ in a history $S$ is called \emph{attackable} if there exists any adversarial set $A\supseteq S\setminus\{(x,y)\}$, $|A|\leq T$ (truthfully labeled), such that the learner would abstain on $x$ given the history $A$.
\end{definition}
We will first bound the total number of attackable samples in any history. Then, we will use the exchangeability property of the i.i.d\ samples to argue that at any time $t$ where $x_t$ is i.i.d., the probability of $x_t$ being attackable is low. Using linearity of expectation can then bound the total number of i.i.d\ samples that the learner can possibly abstain upon.

\begin{lemma}\label{lem:bound attackable samples}
    For any history $S$ (labeled by the target concept $h\in\C$), there are fewer than $e m\left(\alpha +cT^{k-1}\right)^{1/k}$ attackable samples in $S$ for \Cref{alg:potential-learner}, with the input $f$ being an \good function.
\end{lemma}
\begin{proof}
    Let $H\subseteq S$ be the set of attackable samples.
    By the robustness property of $f$, for every subset $M\subseteq H,|M|=m$, there is some $U\subseteq M, |U|=k$, $(x,y)\in M\setminus U$ such that for any $A\supseteq M\setminus(U\cup\{(x,y)\})$, $|A|\leq T$, labeled by $h$,
    $$f(U; \Vcal(A))-f(U; \Vcal(A)_{x\to{-y}})\geq1$$
    or $U$ is unrealizable in $\Vcal(A)_{x\to -y}.$

    It follows from equation 8 of \cite{TuranNumbersSidorenko} that the number of unique such pairs $U,(x,y)$ is at least 
    $$\frac{|H|-k}{|H|-m+1}\frac{\binom{|H|}{k+1}}{\binom{m-1}{k}}\geq \frac{\binom{|H|}{k+1}}{\binom{m}{k}}=:r$$
    
    By the pigeon hole principle, there is some $(x,y)\in H$ such that there exist subsets $U_1,\dots U_{\lceil r/|H|\rceil}$, where $U_i\subseteq H\setminus \{(x,y)\},|U_i|=k,$ satisfy the robustness property with $(x,y)$. That is, for any $A\supseteq H\setminus(U_i\cup \{(x,y)\}),|A|\leq T$ labeled by $h$,
    $$f(U_i, \Vcal(A))-f(U_i, \Vcal(A)_{x\to{-y}})\geq1$$
    or $U_i$ is not realizable in $\Vcal(A)_{x\to -y}.$ We will use this to help show that there can not be too many attackable samples.

    Recall that because $x$ is attackable, there exists some $A\supseteq H\setminus\{(x,y)\},|A|\leq T$ labeled by $h$, the learner will abstain, which implies that
    \begin{align*}
        \alpha&\geq \rho_f(A)-\rho_f(A\cup\{(x,-y)\})
        \intertext{Then, we expand $\rho_f$:}
        \alpha&\geq\sum_{\substack{U\subseteq A\\|U|=k}} f(U; \Vcal(A\setminus U))-\sum_{\substack{U\subseteq A\cup\{(x,-y)\}\\|U|=k}} f(U; \Vcal((A\cup \{(x,-y)\})\setminus U))\\
        \alpha&\geq\sum_{\substack{U\subseteq A\\|U|=k}} \Big(f(U; \Vcal(A\setminus U))-f(U; \Vcal((A\cup \{(x,-y)\})\setminus U))\Big)-\sum_{\substack{U\subseteq A\cup\{(x,-y)\}\\|U|=k\\(x,-y)\in U}} f(U; \Vcal(A\setminus U))
        \intertext{Notice that, assuming the learner abstains on $x$ given the history $A$, each term in the first summation is non-negative (by monotonicity of $f$). Furthermore, each term for which $U=U_i$ for any $i\in[\lceil r/|H|\rceil]$, is at least $1$, by robustness of $f$, since $A\setminus U \supseteq H\setminus (U\cup\{(x,y)\})$ and $U$ must be realizable in $\Vcal(A\setminus U)_{x\to -y}$ (since the learner abstained). So the first summation is at least $r/|H|=\frac{\binom{|H|}{k+1}}{|H|\binom{m}{k}}$.}
        \alpha&\geq \frac{\binom{|H|}{k+1}}{|H|\binom{m}{k}} -\sum_{\substack{U\subseteq A\cup\{(x,-y)\}\\|U|=k\\(x,-y)\in U}} f(U; \Vcal(A\setminus U))
        \intertext{With the remaining summand, we use the fact that $f$ is at most $c$:}
        &\geq \frac{\binom{|H|}{k+1}}{|H|\binom{m}{k}} -\sum_{\substack{U\subseteq A\cup\{(x,-y)\}\\|U|=k\\(x,-y)\in U}} c\\
        &\geq \frac{\binom{|H|}{k+1}}{|H|\binom{m}{k}} - c\binom{|A|}{k-1}\\
        &\geq \frac{\binom{|H|}{k+1}}{|H|\binom{m}{k}} - c\binom{T}{k-1}
        \intertext{Using the inequalities $\left(\frac{n}{k}\right)^k\leq \binom{n}{k}\leq \left(\frac{n}{k}\right)^ke^{k}$ \citep{vershynin2009high},}
        &\geq \frac{|H|^k}{m^{k}e^{k}} - c\binom{T}{k-1}\\
        \implies |H|&\leq  \left(m^{k}e^k
        \left(\alpha+c\binom{T}{k-1}\right)\right)^{1/k} \le em\left(\alpha +c\binom{T}{k-1}\right)^{1/k}
    \end{align*}
    This means that in any history, the number of attackable samples is at most $e m\left(\alpha +c\binom{T}{k-1}\right)^{1/k}$.
\end{proof}
Now we can bound abstention error.

\begin{lemma}\label{lem:abstention}
        Fix an \good function $f$ and run \Cref{alg:potential-learner} with threshold $\alpha>0$ for $T$ rounds in the injection model.
    \begin{align*}
        \mathbb{E}[\errabs] \leq em\left(\alpha +cT^{k-1}\right)^{1/k}\log T
    \end{align*}
\end{lemma}

\begin{proof}
    Let $t_i$ be the time $i^{th}$ time $t$ such that $q_t=0$, unless no such time exists, in which case $t_i=T+i$. Then define for all $i>0$, $x_{T+i}\sim \D,$ and $ \hat{y}_{T+i}=y_{T+i}$ (so that we don't over count abstentions).\footnote{It is important that we add these extraneous `fake' rounds for a subtle reason, formally the i.i.d.\ samples are technically not exchangeable. They are not exchangeable because the number of i.i.d. samples given to the learner can depend on which samples are drawn (and in what order). Consider an adversary that will stop giving i.i.d. samples once they see a certain i.i.d. sample, in this case that sample is not exchangeable with an earlier one, it can only appear as the final i.i.d. sample. }

    \begin{align*}
        \mathbb{E}[\errabs] &= \mathbb{E}\left[\sum_{t=1}^T \mathbf{1}[\hat{y}_t=\perp\land q_t=0]\right]\\
        &= \mathbb{E}\left[\sum_{i=1}^T \mathbf{1}[\hat{y}_{t_i}=\perp]\right]\\
        &= \sum_{i=1}^T Pr\left[\hat{y}_{t_i}=\perp\right]
        \intertext{Notice that if $x_{t_i}$ is not attackable given the history $x_{t_1},\dots x_{t_i}$, then the learner will not abstain on $x_{t_i}$.}
        \mathbb{E}[\errabs]&\leq \sum_{i=1}^T Pr\left[x_{t_i}\text{ is attackable given }\{x_{t_1},\dots x_{t_i}\}\right]
        \intertext{Since $x_{t_1},\dots x_{t_T}$ are sampled i.i.d., they are exchangeable, $\Pr[x_t$ is attackable given $\{x_{t'}\sep t'\in[t], q_{t'}=0\}]$ is the number of attackable samples in $x_{t_1},\dots x_{t_i}$ divided by $i$; using \Cref{lem:bound attackable samples} we can bound this probability to be at most $\frac{em\left(\alpha +cT^{k-1}\right)^{1/k}}{i}.$}
        \mathbb{E}[\errabs]&\leq \sum_{i=1}^T\frac{em\left(\alpha +cT^{k-1}\right)^{1/k}}{i}\\
        &\leq em\left(\alpha +cT^{k-1}\right)^{1/k}\log T
    \end{align*}
    
\end{proof}

\subsection{Proof of \Cref{thm:upper bound}}
\begin{proof}
    The individual error bounds follow directly from \Cref{lem:misclassification} and \Cref{lem:abstention}. We can asymptotically bound the sum of the expected errors by choosing $\alpha=\mathcal{O}\left(\left(\frac{c T^{k}}{k!m \log T}\right)^{\frac{k}{k+1}}\right)$:
    \begin{align*}
        \mathbb{E}[\errmis+\errabs] &\leq \frac{c}{k!\alpha}T^k + em(T^{k-1}+\alpha)^{1/k}\log T\\
        &\leq \left(\frac{c}{k!}\right)^{\frac{1}{k+1}}\left(mT\log T\right)^{\frac{k}{k+1}} + e\left(\frac{c}{k!}\right)^{\frac{1}{k+1}}\left(mT\log T\right)^{\frac{k}{k+1}}\\
        &=\mathcal{O}\left(\left(\left(\frac{c}{k!}\right)^{1/k}mT \log T\right)^{\frac{k}{k+1}}\right)\\
        &=\tilde{\mathcal{O}}\left( \frac{c^{\frac{1}{k+1}}}{k} m^{\frac{k}{k+1}}T^{\frac{k}{k+1}}\right)
    \end{align*}
\end{proof}
\section{Proofs for \Cref{subsec:prior art}}
\subsection{Proof of \Cref{lem:axis aligned rectangles}}
\begin{proof}
    In this proof, let $f$ refer to $f_{\text{rect}}$, defined in \ref{def:rect-score}.
    We will define $\delta_i(x)$ to be the same as $x$ in dimension $i$, and is $0$ elsewhere. $\delta_i$ in the lemma statement refers to $\delta_i(x)$.
    First note that $f$ has the range $[0,d]$.

    \textbf{Monotonicity}. Fix a realizable singleton $U=\{(x,y)\}$ and a restriction $V' = V_{x'\to y'}$. Then $V'_{x\to y} = V_{x\to y, x'\to y'} \subseteq V_{x\to y}$. Hence for every $i$, if $(\delta_i(x),-1)$ is realizable in $V_{x\to y, x'\to y'}$ then it is realizable in $V_{x\to y}$. Therefore the counted set of coordinates can only shrink, and $f(U;V) \ge f(U;V')$.

    \textbf{Robustness.} Fix any realizable $M\subseteq \Lcal, |M|=2d+1$ labeled by a target rectangle $h$. Consider the largest set $S$ whose elements all have the same label $y$ for some $y \in \{\pm1\}$. By pigeonhole, $|S| \ge d+1$. For each $i\in[d]$, choose $(x_i,y)\in S$ maximizing the $i$-th coordinate among points in $S$. Choose $(s,y)\in S\setminus\{(x_1,y),\ldots,(x_d,y)\}$ (possible since $|S|\ge d+1$).
    \begin{enumerate}
        \item \emph{Case 1: $y=+1$.}  Let $U$ be any single sample in $M \setminus \big(\{(s,+1)\}\cup\{(x_i,+1):i\in[d]\}\big)$. Such a $U$ exists because $|M| = 2d + 1$. Let the witness be $(s,+1)$. Now fix any realizable extension $A \supseteq M \setminus (U \cup \{(s,+1)\})$ labeled by $h$, and let $V=\Vcal(A)$. Since $A$ contains every $(x_i,+1)$, any axis-aligned rectangle $h'$ parameterized by thresholds $b_1, \ldots, b_d$ consistent with $A$, it must have $b_i \ge (x_i)_i$ for all $i$, and therefore $b_i \ge s_i$ for all $i$. Hence $h'(s)=+1$, so $V_{s\to -1}=\emptyset$. So, $U$ is unrealizable in $V_{s\to -1}$.
        \item \emph{Case 2: y=-1.} Since $h(s)=-1$, there exists some coordinate $i\in[d]$ with $h(\delta_i(s))=-1$. Let $U=\{(s,-1)\}$ and let the witness be $(x_i,-1)$. Fix any realizable extension $A \supseteq M \setminus (U \cup \{(x_i,-1)\})$ labeled by $h$ and let $V=\Vcal(A)$. Because $h\in V_{s\to -1}$ and $h(\delta_i(s))=-1$, the coordinate $i$ is counted in $f(U;V)$. On the other hand, in $V_{s\to -1, x_i\to +1}$, any rectangle $h'$ parameterized by thresholds $b_1, \ldots, b_d$ with $h'(x_i)=+1$ must have $b_i \ge (x_i)_i \ge s_i$, hence $h'(\delta_i(s))=+1$. Therefore $(\delta_i(s),-1)$ is not realizable in $V_{s\to -1, x_i\to +1}$, so coordinate $i$ is not counted after the restriction $x_i\to +1$. Since $V_{s\to -1,\,x_i\to +1}\subseteq V_{s\to -1}$, any coordinate counted after the restriction was already counted before, hence showing that $i$ is not counted after the restriction implies the score drops by at least $1$.
    \end{enumerate}
    This shows $f$ is \good[2d+1][d][1]
\end{proof}

\subsection{Proof of \Cref{lem: vc1}}
As in \cite{goel2024adversarialresiliencesequentialprediction}, we leverage the fact (due to \citep{bendavid20152notesclassesvapnikchervonenkis}) that any VC-$1$ class can be represented as the class of initial segments in a tree ordering.
\begin{definition}
    For a domain $\X$, with a partial order $\prec$, for any $x\in \X$, we call the set $\{x'\in \X \sep x'\preceq x\}$ the \emph{initial segment} of $x$ under $\prec$. We call $\prec$ a \emph{tree ordering} if for any initial segment $I$, for any $x,x'\in I$, $x\prec x'$ or $x'\prec x$, or $x=x'$.

\end{definition}
\begin{definition}
    For any concept $r\in \C$. For each $h\in \C$, define $h_r: \X\to\{\pm 1\}$ by $$h_r(x)=\begin{cases}   +1 & h(x)\neq r(x)\\-1&h(x)= r(x)\end{cases}.$$
    Let $\C_r:= \{h_r: h \in \C\}$.
\end{definition}
\begin{theorem}[\citep{bendavid20152notesclassesvapnikchervonenkis}]\label{thm:vc 1 structure}
    Let $\C$ be a concept class over some domain $\X$. Then the following statements are equivalent
    \begin{enumerate}
        \item $\VCdim(\C)\leq 1$
        \item For every $r\in \C$, there exists some tree ordering over $\X$ such that the pre-image of $+1$ in any element of $\C_r$ is an initial segment in $\prec$.
    \end{enumerate}
\end{theorem}
To prove \Cref{lem: vc1}, we use \Cref{lem: tree order}.
\subsubsection{Proof of \Cref{lem: tree order}}
\begin{proof}
    We will show that $f_{\mathrm{seg}}$ is a \good[3][1][1] function, and then the desired combined error bound follows directly from \Cref{thm:upper bound}.
    
    First note that the range of $f_{\mathrm{seg}}$ is $\{0,1\}$.
    
    \textbf{Monotonicity.} Fix $U=\{(x,y)\}$ and a version space $V$ consistent with $U$ (so $V_{x\to y}\neq\emptyset$). Let $(x',y')\in\Lcal$ be realizable with $U$ in $V$ and set $V':=V_{x'\to y'}\subseteq V$. Then, $P(V)\subseteq P(V')$ (restricting a version space can only add forced positives), so $\mathbf{1}[P(V)\subseteq I(x)]\ge \mathbf{1}[P(V')\subseteq I(x)]$. Hence $f_{\mathrm{seg}}(U;V)\ge f_{\mathrm{seg}}(U;V')$.

    \textbf{Robustness.} Let $M\subseteq \Lcal$ with $|M|=3$ be consistent with some $h\in\C$. We show there exist $U\subseteq[M], |U|=1$ and a witness $(w,y_w)\in M\setminus U$ such that for every realizable extension $A\supseteq M\setminus(U\cup\{(w,y_w)\})$, letting $V=\Vcal(A)$, either
    \[
    f_{\mathrm{seg}}(U;V)-f_{\mathrm{seg}}(U;V_{w\to -y_w})\ge 1.
    \] or $V_{w\to -y_w}=\emptyset$.
    Let the three samples in $M$ be $(x_1,y_1),(x_2,y_2),(x_3,y_3)$.
    Assume WLOG $y_1=y_2$. We break the analysis into three cases.
    \begin{enumerate}
        \item \emph{Case 1: $y_1=y_2=+1$} Assume WLOG $x_1\prec x_2$. Choose $U=\{(x_2,+1)\}$ and witness $(w,y_w)=(x_1,+1)$. Let $A$ be any realizable extension with $A \supseteq M \setminus (U \cup \{(x_1,+1)\})$ and let $V=\Vcal(A)$. Then $(x_2,+1)$ is realizable in $V$ (since $h$ realizes $A$), so $f_{\mathrm{seg}}(U;V)\ge 0$. Moreover, in $V_{x_1\to -1}$ no hypothesis can label $x_2$ as $+1$ because any initial segment containing $x_2$ must also contain $x_1$. Thus $U$ is unrealizable in $V_{x_1\to -1}$.
        \item \emph{Case 2: $y_1=y_2=-1$, and $x_1, x_2$ are comparable.} Assume WLOG $x_1\prec x_2$. Choose $U=\{(x_1,-1)\}$ and witness $(w,y_w)=(x_2,-1)$. Let $A$ be any realizable extension with $A \supseteq M \setminus (U \cup \{(x_2,-1)\})$ and let $V=\Vcal(A)$. But if $x_2$ is labeled $+1$, then we must have $x_1$ labeled $+1$ as well, so $U$ is unrealizable in $V_{x_2\to +1}$.
        \item \emph{Case 3: $y_1=y_2=-1$, and $x_1,x_2$ are incomparable.}
For the sake of contradiction, assume that there exist realizable extensions
$A'_1, A'_2$ (labeled by $h$) such that for $i\in\{1,2\}$,
\begin{align*}
    P(\Vcal(A'_i))\subseteq I(x_i),\qquad
    P(\Vcal(A'_1))\not\subseteq I(x_2),\qquad
    P(\Vcal(A'_2))\not\subseteq I(x_1).
\end{align*}
Since $h\in \Vcal(A'_1)\cap \Vcal(A'_2)$, each $P(\Vcal(A'_i))$ is an initial segment
contained in $h^{-1}(+1)$, hence $P(\Vcal(A'_1))$ and $P(\Vcal(A'_2))$ are comparable
by inclusion. Assume WLOG $P(\Vcal(A'_1))\subseteq P(\Vcal(A'_2))$. Then
$P(\Vcal(A'_1))\subseteq P(\Vcal(A'_2))\subseteq I(x_2)$, contradicting
$P(\Vcal(A'_1))\not\subseteq I(x_2)$. Therefore, after possibly swapping indices,
we may assume that for any realizable extension $A'$ (labeled by $h$),
\[
P(\Vcal(A'))\subseteq I(x_2)\implies P(\Vcal(A'))\subseteq I(x_1).
\]

Choose $U=\{(x_1,-1)\}$ and witness $(w,y_w)=(x_2,-1)$, and fix any realizable extension
$A\supseteq M\setminus(U\cup\{(x_2,-1)\})$ labeled by $h$. Let $V=\Vcal(A)$.

If $V_{x_2\to +1}=\emptyset$, then $U$ is unrealizable in $V_{x_2\to +1}$ and we are done.

Otherwise, if $f_{\mathrm{seg}}(U;V)=0$, then $P(V)\not\subseteq I(x_1)$ by definition of
$f_{\mathrm{seg}}$. By the contrapositive of the implication above, $P(V)\not\subseteq I(x_2)$,
so pick $x'\in P(V)\setminus I(x_2)$. Since $x'\in P(V)$ we have $h(x')=+1$, while
$h(x_2)=-1$. If $x_2\prec x'$, then initial-segment closure would force $h(x_2)=+1$,
a contradiction; hence $x'$ is incomparable with $x_2$. But any hypothesis in $V_{x_2\to +1}$
must label $x'$ as $+1$ as well (since $x'\in P(V)$), and no initial segment can contain two
incomparable points, so $V_{x_2\to +1}=\emptyset$, and $U$ is unrealizable in $V_{x_2\to +1}$.

On the other hand, if $f_{\mathrm{seg}}(U;V)=1$, then $P(V)\subseteq I(x_1)$. In the restricted
space $V_{x_2\to +1}$, the point $x_2$ is forced positive, i.e., $x_2\in P(V_{x_2\to +1})$.
Since $x_2\notin I(x_1)$ (as $x_1,x_2$ are incomparable), we have
$P(V_{x_2\to +1})\not\subseteq I(x_1)$ and hence $f_{\mathrm{seg}}(U;V_{x_2\to +1})=0$.
Therefore,
\[
f_{\mathrm{seg}}(U;V)-f_{\mathrm{seg}}(U;V_{x_2\to +1})=1-0=1.
\]
    \end{enumerate}
\end{proof}
Now we can prove \Cref{lem: vc1}.
\begin{proof}
    Fix some reference function $r\in \C$ and corresponding concept class $\C_r$. Since $VCdim(\C) \le 1$, \Cref{thm:vc 1 structure} implies there exists a tree ordering $\prec$ on $\X$ such that for every $h \in \C_r$, the pre-image of $+1$ under $h$ is an initial segment with respect to $\prec$.

    For any version space $V\subseteq \C$, define $V_r := \{h_r : h\in V\}\subseteq \C_r$. For any labeled sampled $(x,y)\in \Lcal$, define
    \[
    y_r(x,y):= \begin{cases}
                    +1 & \text{if } y\neq r(x),\\
                    -1 & \text{if } y= r(x).
                \end{cases}
    \]
    For any labeled set $S\subseteq \Lcal$, define $S_r := \{(x, y_r(x,y)) : (x,y)\in S\}$.
    Define the \good[3][1][1] score function for $\C$ by
    \begin{align*}
    f_{\mathrm{vc1}}(\{(x,y)\};V) \;:=\; f_{\mathrm{seg}}(\{(x,y_r(x,y))\};V_r),
    \end{align*}
    where $f_{\mathrm{seg}}$ is the function from \Cref{lem: tree order} applied to the initial-segment class $\C_r$ under $\prec$. 

    We claim that the map $h\mapsto h_r$ is a bijection $\C\to\C_r$ and preserves consistency in the following sense: for any $V\subseteq\C$, any $(x,y)\in\Lcal$, and $V' := V_{x\to y}$,
    \[
    (V')_r \;=\; (V_r)_{x\to y_r(x,y)}.
    \]
    Indeed, for any $h\in\C$ we have $h(x)=y$ iff $h_r(x)=y_r(x,y)$ by definition of $y_r$ and $h_r$. Since $f_{\mathrm{seg}}$ is \good[3][1][1] for $\C_r$ by \Cref{lem: tree order}, and the relabeling is a bijection commuting with restrictions, the three properties (boundedness on realizable inputs, monotonicity, robustness) transfer from $f_{\mathrm{seg}}$ to $f$. 
    Therefore $f$ is \good[3][1][1] for $\C$. Applying \Cref{thm:upper bound} with $\alpha=\sqrt{T}$ yields
    \[
    \mathbb{E}[err_{mis}+err_{abs}] \le \tilde{\mathcal{O}}\!\left(T^{1-\frac{1}{1+1}}\right)=\tilde{\mathcal{O}}(\sqrt{T}).
    \]
\end{proof}
\section{Proofs for \Cref{sec:upper_bounds}}\label{sec:proof of upper bounds}
\subsection{Proof of \Cref{lem:certificate dimension}}
Define the function $N$ from subsets of $\L$ to equal sized subsets of $\X$ as
\begin{align*}
    N(S)=\{x\sep \exists (x,y)\in S\}
\end{align*}
\begin{proof}
    For any $\C,\X$ with certificate dimension $k$ with parameters $(m,n)$, let $\sigma$ be as it is in the definition of certificate dimension.

    For any $S\subseteq \Lcal, |S|=k-1$, let $S$ and any version space $V$, let
    $$g(N(S);V)=\{i\in\{0,1,\dots n-1\}\sep \exists h\in V, \sigma(h,N(S))=i\},$$
    and
    $$f(S;V)=|g(N(S);V)|.$$
    First note that $f$ is bounded between $0$ and $n$ on all inputs.

    Furthermore, for any $S\subseteq \Lcal, |S|=k-1$, version space $V$ consistent with $S$, and any $(x,y)$ realizable in $V$, 
    \begin{equation}\label{eq:certificate subset}
        g(N(S);V_{x\to y}) \subseteq g(N(S);V) 
    \end{equation}
    since $h\in V_{x\to y}\implies h\in V$. So, monotonicity is satisfied:
    $$f(S;V)\geq f(S;V_{x\to y})$$

    Finally, for any $h\in\C$, $S\subseteq \Lcal, |S|=m$, by the definition of certificate dimension, there is some $S'\subseteq N(S)$, $|S'|=k$, and $x\in S'$, such that for all $h'\in \Vcal(S\setminus\{(x,y)\})$, $$\sigma(h',S'\setminus\{x\})=\sigma(h,S'\setminus\{x\})\implies h'(x)=h(x).$$
    Consider any $U\supset S$, consistent with $h$, then $$\sigma(h, S'\setminus\{x\})\in g(S'\setminus\{x\};\Vcal(U)),$$  however, restricting the version space to $x\to-y$ eliminates $\sigma(h, S'\setminus\{x\})$ from the set,
    $$\sigma(h, S'\setminus\{x\})\notin g(S'\setminus\{x\};\Vcal(U)_{x\to -y})$$
    Combined with \Cref{eq:certificate subset}, we get that
    $$f(S',\Vcal(U))-f(S',\Vcal(U)_{x\to -y})\geq1$$
    Proving that $f$ is \good[m][n][k-1].
\end{proof}

\subsection{Proof of \Cref{lem:inference dimension}}
\begin{proof}
    For any $\C,\X$ with inference dimension $k$, for each distinct comparison-query transcript, associate a number, starting from $0$, going to $r-1$ (where $r$ is the total number of distinct transcripts). Consider the certificate function $\sigma(h,S)$ that returns the number associated with $Q_h(N(S))$. By the definition of inference dimension, for any set $S\subseteq \Lcal$ of size $k$, $h\in \C$, $\exists (x,y)\in S$ such that $N(S)\setminus\{x\}$ infers $x$ under $h$, meaning that for all $h'\in \C\supseteq \Vcal(S)$, $$\sigma(h', N(S)\setminus \{x\})=\sigma(h, N(S)\setminus \{x\})\implies h'(x)=h(x)$$
    This proves that $\C,\X$ have certificate dimension $k$, with parameters $m=k$ and $n=r$.

    All that remains to do is upper bound $r$. Notice that $r$ equivalently counts the number of ways to order the elements of $\{0\}\cup \{f_h(x)\sep x\in N(S)\}$ and place either $\geq$ or $=$ between each adjacent elements. The number of ways to just order them is $k!$, and there are $k-1$ spaces to place $\geq$ or $=$, so for any fixed order, there are $2^{k-1}$ ways to answer the queries. This gives an upper bound of $r\leq 2^{k-1} k!$.

    Finally, we apply \ref{lem:certificate dimension}, to show that the learner can obtain a combined error of
    \begin{align*}
        \tilde{\mathcal{O}}\left(\frac{1}{k}m n^{1/k}T^{1-1/k}\right)=
        \tilde{\mathcal{O}}\left( (2^{k-1} k!)^{1/k}T^{1-1/k}\right)=
        \tilde{\mathcal{O}}\left(k T^{1-1/k}\right)
    \end{align*}
    Completing the proof.
\end{proof}

\subsection{Proof of \Cref{thm:halfplanes}}\label{subsec:proof of halfplanes}
To prove \Cref{thm:halfplanes}, we first consider a slight modification of two dimensional halfspaces.
For any $s_+,s_-\in \mathbb{R}^2$, let $\C=\{h(x)=\sign(w^T x - b)\sep w\in \mathbb{R}^2,b\in\mathbb{R}, w^\top s_+ \geq b> w^\top s_-\}$ be two dimensional $\{(s_+,+1),(s_-,-1)\}$-halfspaces. In \Cref{lem:halfplane helper} will show that two dimensional $\{(s_+,+1),(s_-,-1)\}$-halfspaces has certificate dimension $3$ (with parameters $m=7$ and $n=3$).

Now we are ready to prove \Cref{thm:halfplanes}.
\begin{proof}
    We will use a learner that first predicts $\hat{y}_1=+1$, then as long as the history only contains points of one label will predict that label. This process results in at most two misclassifications, and no abstentions. If at some time $t'$, the history contains points of both labels, the learner will choose any $(s_+,+1)$ and $(s_-,-1)$ from the history. Then, by \Cref{lem:halfplane helper} and \Cref{lem:certificate dimension}, there is a learner that, on the remaining $T-t'$ rounds, has expected combined error $\tilde{\mathcal{O}}\left((T-t')^{1-1/3}\right)$. Therefore, the learner overall achieves
    \begin{align*}
        \mathbb{E}[\errmis +\errabs]\leq \tilde{\mathcal{O}}\left(T^{2/3}\right) 
    \end{align*}
\end{proof}

\begin{lemma}\label{lem:halfplane helper}
    For any $s_+, s_-\in \mathbb{R}^2$, define the concept class two dimensional $\{(s_+,+1),(s_-,-1)\}$-halfspaces as
    $$\left\{h(x)=\sign\left(w^\top x- b\right)\sep \forall w\in \mathbb{R}^2, b\in \mathbb{R}, \text{ such that }w^\top s_+ \geq b> w^\top s_-\right\}$$
    For any distinct $s_+,s_-\in \mathbb{R}^2$, two dimensional $\{(s_+,+1),(s_-,-1)\}$-halfspaces has certificate dimension $3$ (with parameters $m=7$ and $n=3$).
\end{lemma}
The proof leverages the well known Radon's theorem (1.3.1 Theorem in \cite{Matousek2002}), that shows, as a special case, that for any set of $4$ points $\mathbb{R}^2$, there exists some partition of the points into two subsets whose convex hulls intersect.
\begin{proof}
    Throughout the proof, we will use the fact that the set of points with a particular label under a halfspace is convex.
    
    Our certificate will be $$\sigma(h,\{(x_a,y_a),(x_b,y_b)\})=\begin{cases}
        0 & w_h^\top x_a> w_h^\top x_b\\
        1 & w_h^\top x_a= w_h^\top x_b\\
        2 & w_h^\top x_a< w_h^\top x_b
    \end{cases}$$ where $w_h\in \mathbb{R}^2$ is defined such that for some $b\in \mathbb{R}$, $h(x)=\mathbf{1}[w_h^\top x\geq b],$ and $x_a,x_b$ are labeled lexicographically by their coordinates so that $x_b$ is lexicographically greater than or equal to $x_a$.

    In any $M\subseteq \Lcal$, $|M|=7$, there must be a subset of size $4$ in $M$ all sharing the same label. Without loss of generality, let this label be $+1$, and name the points $x_0,x_1,x_2,x_3\in \mathbb{R}^2$.

    To finish the proof we must show that the label of one of the points, $x_i$, is determined by $\sigma(h, \{x_j, x_k\})$ in $\Vcal(\{(x_l,+1)\sep l\neq i\})$ for distinct $i,j,k\in\{0,1,2,3\}$.

    If one of the points $x_i$ falls within the convex hull of the others, then for all $h'\in \Vcal(\{(x_j, +1)\sep j\neq i\})$, $h'(x_i)=+1$. Now we will consider the case where $x_0,x_1,x_2,x_3$ are in convex position. By Radon's theorem, without loss of generality, the line segments formed by $x_0,x_1$ and $x_2,x_3$ intersect. Still without loss of generality, assume that $w_h^\top x_0\leq w_h^\top x_i$, and that $x_2$ is on the same side of the (infinite) line defined by $x_0,x_1$ as $s_-$. 

    Applying Radon's theorem to  $x_0,x_1,x_2,s_-$, we obtain a partition into sets $S,U$. Without loss of generality, assume $|S|\geq |U|$. We will break the analysis up into cases depending on $U$, showing that each case is either impossible, or satisfies our desired claim.

    \begin{enumerate}
        \item $U=\{s_-\}$. This is impossible, since the convex hull of $\{x_0,x_1,x_2\}$ can contain only points labeled $+1$ by $h$.
        \item $|U|=2$. Without loss generality, define $i$ and $U$ such that $U=\{s_{-}, x_i\}$. In this case, the label of $x_i$ is determined in $\Vcal(\{(x_l,+1)\sep l\neq i\})$, because if $h'(x_i)=-1$, then the line segment defined by $U$ would be all labeled $-1$, but this intersects with the line segment defined by $S$, which must be labeled $+1$. 
        \item $U=\{x_i\}$ for some $i\in \{0,1\}$. This is also impossible, since $x_2,s_-$ are on the same side of the infinite line defined by $x_0,x_1$, the convex hull of $\{x_{1-i}, x_2,s_-\}$ can only intersect the line segment $x_0,x_1$ at $x_{1-i}$, so it can not contain $x_i$. The only exception to this is if $x_0,x_1,s_-$ are collinear, in which case the label of $x_{1-i}$ is determined in $\Vcal(\{(x_l,+1)\sep l\neq i\})$, since $x_i$ would be in the convex hull of $x_{1-i}$ and $s_{-}$.
        \item $U=\{x_2\}$. We will show that $x_1$ is determined by $\sigma(h ,\{(x_0,+1), (x_2,+1)\})$ in $$V=\Vcal(\{(x_0,+1),(x_2,+1),(x_3,+1)\}).$$ Consider any $h'\in V_{x_1\to -1}$. Since $U=\{x_2\}$, we have $x_2\in \mathrm{conv}(\{x_0,x_1,s_-\})$.
        Thus there exist $\alpha_0,\alpha_1,\alpha_s\ge 0$ with
        $\alpha_0+\alpha_1+\alpha_s=1$ and $\alpha_1+\alpha_s>0$ such that
        $x_2=\alpha_0 x_0+\alpha_1 x_1+\alpha_s s_-$.
        
        Because $h'\in V_{x_1\to -1}$, we have
        $w_{h'}^\top x_0\ge b_{h'}$, $w_{h'}^\top x_1<b_{h'}$, and $w_{h'}^\top s_-<b_{h'}$.
        Therefore,
        \begin{align*}
            w_{h'}^\top x_2
            &=\alpha_0 w_{h'}^\top x_0+\alpha_1 w_{h'}^\top x_1+\alpha_s w_{h'}^\top s_- \\
            &< \alpha_0 w_{h'}^\top x_0 + (\alpha_1+\alpha_s) b_{h'} \\
            &\le \alpha_0 w_{h'}^\top x_0 + (\alpha_1+\alpha_s) w_{h'}^\top x_0 \\
            &=(\alpha_0+\alpha_1+\alpha_s) w_{h'}^\top x_0
            = w_{h'}^\top x_0.
        \end{align*}
        So $\sigma(h',\{(x_0,+1), (x_2,+1)\})\neq \sigma(h ,\{(x_0,+1), (x_2,+1)\})$, regardless of the lexicographic ordering of $x_0, x_2$. This means that the label of $x_1$ is determined by the certificate $\sigma(h, \{(x_0,+1), (x_2,+1)\})$.
    \end{enumerate}
\end{proof}
\end{document}